\documentclass{tlp}
\usepackage{amsmath,amssymb}
\usepackage{xspace}
\usepackage{url}
\usepackage{color}
\usepackage{fancyvrb}
\def\nec{\Box}
\def\pos{\Diamond}
\def\Next{\bigcirc}
\newcommand{\V}[1]{\mathbf{#1}}

\def\M{\mathcal{M}}

\def\H{\V{H}}
\def\T{\V{T}}

\def\D{\V{D}}
\def\L{\mathcal{L}}
\newcommand{\tuple}[1]{\langle #1 \rangle}
\newcommand{\res}[1]{\mathord{|}_{#1}\xspace}
\newcommand{\infest}{\mathbin{|\kern -.47em\sim}}
\newcommand{\LDef}{\ensuremath{\Lang=\langle C,P \rangle}\xspace}
\newcommand{\Lang}{\ensuremath{\mathcal{L}}\xspace}
\newcommand{\eqdef}{\mathbin{\stackrel{\mathrm{def}}{=}}}

\newcommand{\unmedio}{\leavevmode
\kern.1em\raise.5ex\hbox{$\scriptstyle1$}
\kern-.13em\hbox{$\scriptstyle/$}\kern-.1em\lower.25ex\hbox{$\scriptstyle2$}}

\def\qed{~\hfill{$\boxtimes$}}
\newenvironment{proofof}[1]{\vspace{10pt}\noindent {\bf Proof of~#1}.}{\qed}

\newtheorem{definition}{Definition}
\newtheorem{lemma}{Lemma}
\newtheorem{proposition}{Proposition}
\newtheorem{theorem}{Theorem}
\newtheorem{example}{Example}

\begin{document}

\long\def\comment#1{{\color{blue} #1}}

\title{Temporal Logic Programs with Variables}
\author[F. Aguado, P. Cabalar, M. Di{\'e}guez, G. P{\'e}rez, C. Vidal]{
    Felicidad Aguado$^1$, Pedro Cabalar$^1$, Mart{\'{\i}}n Di{\'e}guez$^2$, Gilberto P{\'e}rez$^1$ \and Concepci{\'o}n Vidal$^1$\\
    \\
    $^1$ Dept. of Computer Science\\
    University of Corunna, SPAIN\\
    \texttt{\{aguado,cabalar,gperez,eicovima\}@udc.es}\\ \\
    $^2$ IRIT - Universit\'e Paul Sabatier\\
    Toulouse, FRANCE\\
    \texttt{martin.dieguez@irit.fr}\\    
}

\pagerange{\pageref{firstpage}--\pageref{lastpage}}
\jdate{January 2016}
\setcounter{page}{1}
\pubyear{2016}

\maketitle

\label{firstpage}

\begin{abstract}
In this note we consider the problem of introducing variables in temporal logic programs under the formalism of \emph{Temporal Equilibrium Logic} (TEL), an extension of Answer Set Programming (ASP) for dealing with linear-time modal operators. To this aim, we provide a definition of a first-order version of TEL that shares the syntax of first-order Linear-time Temporal Logic (LTL) but has a different semantics, selecting some LTL models we call \emph{temporal stable models}. Then, we consider a subclass of theories (called \emph{splittable temporal logic programs}) that are close to usual logic programs but allowing a restricted use of temporal operators. In this setting, we provide a syntactic definition of \emph{safe variables} that suffices to show the property of \emph{domain independence} -- that is, addition of arbitrary elements in the universe does not vary the set of temporal stable models. 
%Next, we propose a definition of \emph{derivable temporal facts} in a ground program so that, if a fact is not derivable, its negation can be added to the program without varying the set of temporal stable models.
 Finally, we present a method for computing the derivable facts by constructing a non-temporal logic program with variables that is fed to a standard ASP grounder. The information provided by the grounder is then used to generate a subset of ground temporal rules which is equivalent to (and generally smaller than) the full program instantiation.
 
{\bf Under consideration in Theory and Practice of Logic Programming (TPLP)}
\end{abstract}


\section{Introduction}

Many application domains and example scenarios from Answer Set Programming (ASP)~\cite{Nie99,MT99} contain a dynamic component, frequently representing transition systems over discrete time. %However, temporal reasoning in ASP tends to be quite rudimentary, just treating time as an integer variable which is grounded for a finite interval.
%\footnote{More elaborated approaches~\cite{MGZ08} deal with arbitrary temporal distances by using a constraint satisfaction tool as a backend.}
In an attempt to provide a full logical framework for temporal reasoning in ASP, \cite{ACPV08} proposed a formalism called \emph{Temporal Equilibrium Logic} (TEL), syntactically identical to propositional \emph{Linear-time Temporal Logic} (LTL) \cite{Pnueli77}, but semantically relying on a temporal extension of \emph{Equilibrium Logic}~\cite{Pea96}, the most general and best studied logical characterisation of stable models~\cite{GL88}. 
%Some recent results have shown potential utilities of counting on the support of a full logical characterisation. For instance, it was shown how TEL can be naturally used for comparing different representations of a same temporal scenario~\cite{ACPV08} or to prove decidability and assess complexity~\cite{CD11b} for different problems for temporal reasoning in ASP.
 In~\cite{ACPV11} a reduction of (propositional) TEL into regular LTL was presented, but applicable to a subclass of temporal theories called \emph{splittable Temporal Logic Programs}.
% , is a strict subset of the normal form fixed in~\cite{Cab10}. 
This syntactic fragment deals with temporal rules in which, informally speaking, ``past does not depend on the future,'' a restriction general enough to cover most (if not all) existing examples of ASP temporal scenarios. The reduction was implemented in a tool, {\tt STeLP}\footnote{\url{http://kr.irlab.org/stelp_online}}~\cite{CD11a}, that computes the temporal stable models of a given program, that are shown as a B\"uchi automaton.
%(each temporal stable model usually describes an infinite sequence of states).

%Although the theoretical results on which {\tt STeLP} is based are restricted to the propositional case,

Although the approach in~\cite{ACPV11} was exclusively propositional, the input language of {\tt STeLP} \ was extended with variables. As in non-temporal ASP, these were just understood as a shortcut for all their possible ground instances.
%In fact, {\tt STeLP} uses as a backend the ASP solver {\tt DLV}~\cite{LPF+06} for the main part of its grounding process.
This initial approach was not fully satisfactory for several reasons. First, it forced that any variable instance was not only \emph{safe} (that is, occurring in the positive body of the rule) but also ``typed'' by a \emph{static} predicate, i.e., a predicate whose extent does not vary along time. Second, this restriction implied sometimes the generation of irrelevant ground rules that increased the size of the resulting ground LTL theory while they could be easily detected and removed by a simple analysis of the temporal program. Last, but not least, the treatment of variables had no formal background and had not been proved to be sound with respect to the important property of \emph{domain independence}~\cite{BFL08} -- essentially, a program is domain independent when its stable models do not vary under the arbitrary addition of new constants. Although the usual definition of safe variables guarantees domain independence, there was no formal proof for temporal logic programs under TEL.

In this note we provide some results that allow an improved treatment of variables in temporal logic programs, using a first order version of TEL as underlying logical framework. We relax the {\tt STeLP} definition of safe variable by removing the need for static predicates so that, as in ASP, a variable in a rule is \emph{safe} when it occurs in the positive body\footnote{This definition of safety, initially introduced in {\tt DLV}~\cite{LPF+06} has been adopted in the standard ASP-Core-2~\cite{ASPCore2} and also followed by {\tt Gringo}~\cite{GKK+11}.}. We prove that this simpler safety condition satisfies domain independence. 
%(this indirectly proves that the previous, more restrictive definition was also correct). 
Finally, we describe a method for grounding temporal logic programs under this new safety condition that still allows calling a standard ASP grounder as a backend, but using a positive normal logic program that is generated by a given transformation on the original temporal logic program.

%The rest of the paper is organised as follows. In the next section, we explain our motivations using an example. In Section~\ref{sec:qel} we introduce the first order extension of TEL and provide some basic syntactic definitions. Next, we study the relaxed definition of safe variables and prove that it guarantees domain independence. Section~\ref{sec:ground} defines the concept of \emph{derivable facts}, explaining how they can be computed and used afterwards to generate smaller ground theories. Section~\ref{sec:conc} concludes the paper.

%%%%%%%%%%%%%%%%%%%%%%%%%%%%%%%%%%%%%%%%%%%%%%%%%%%%%%%%%%%%%%%%%%%%%%%%%%%%%%%%%%
\section{A motivating example}\label{sec:motiv}
%%%%%%%%%%%%%%%%%%%%%%%%%%%%%%%%%%%%%%%%%%%%%%%%%%%%%%%%%%%%%%%%%%%%%%%%%%%%%%%%%%

%For a better understanding, let us consider a simple illustrative example.

\begin{example}\label{ex:cars}
Suppose we have a set of cars placed at different cities and, at each transition, we can drive a car from one city to another in a single step, provided that there is a road connecting them.\qed
\end{example}

Figure~\ref{fig:cars} contains a possible representation of this scenario in the language of {\tt STeLP}. Operator `{\tt o}' stands for ``next'' whereas ``{\tt ::-}'' corresponds to the standard ASP conditional ``{\tt :-}'', but holding at all time points. Rule {\tt (1)} is the effect axiom for driving car {\tt X} to city {\tt A}. The disjunctive rule {\tt (2)} is used to generate possible occurrences of actions in a non-deterministic way. Rules {\tt (3)} and {\tt (4)} represent the inertia
%\footnote{Auxiliary predicates {\tt no\_driveto(X,B)} and {\tt no\_at(X,A)} play the role of strong negation.} 
of fluent {\tt at(X,A)}. Finally, rule {\tt (5)} forbids that a car is at two different cities simultaneously.
\begin{figure}[htbp]
\begin{center}
\begin{Verbatim}[frame=single,fontsize=\footnotesize]
static city/1, car/1, road/2.

o at(X,A)  ::- driveto(X,A), car(X), city(A).                       % (1)

driveto(X,B) v no_driveto(X,B) ::- at(X,A), car(X), road(A,B).      % (2)

o at(X,A)    ::- at(X,A), not o no_at(X,A), car(X), city(A).        % (3)
  no_at(X,A) ::- at(X,B), A!=B, car(X), city(A), city(B).           % (4)

::- at(X,A), at(X,B), A!=B, car(X), city(A), city(B).               % (5)
\end{Verbatim}
\end{center}
\caption{A simple car driving scenario.}
\label{fig:cars}
\end{figure}

As we can see in the first line, predicates {\tt city/1}, {\tt car/1} and {\tt road/2} are declared to be {\tt static}. The scenario would be completed with rules for static predicates. These rules constitute what we call the \emph{static program} and can only refer to static predicates without temporal operators. An example of a static program for this scenario could be:
\begin{verbatim}
road(A,B) :- road(B,A).  % roads are bidirectional
city(A) :- road(A,B).
car(1). car(2).
road(lisbon,madrid). road(madrid,paris).
road(boston,ny). road(ny,nj).
\end{verbatim}
\noindent Additionally, our temporal program would contain rules describing the initial state like, for instance, the pair of facts:

\begin{verbatim}
at(1,madrid). at(2,ny).
\end{verbatim}

Note that all variables in a rule are always in some atom for a static predicate in the positive body. 
%This sometimes makes rule bodies quite long and slightly redundant. 
The current grounding process performed by {\tt STeLP} just consists in feeding the static program to an ASP grounder ({\tt DLV} or {\tt gringo}) and, once it provides an extension for all the static predicates, each temporal rule is instantiated for each possible substitution of variables according to static predicates. In our running example, for instance, the grounder provides a unique model\footnote{If the static program yields several stable models, each one generates a different ground theory whose temporal stable models are computed independently.} for the static program containing the facts:
\begin{verbatim}
car(1), car(2), city(lisbon), city(madrid), city(paris), 
city(boston), city(ny), city(nj), road(lisbon,madrid), 
road(madrid,lisbon), road(madrid,paris), road(paris,madrid), 
road(boston,ny), road(ny,boston), road(ny,nj), road(nj,ny)
\end{verbatim}
\noindent With these data, rule {\tt (1)} generates 12 ground instances, since we have two possible cars for {\tt X} and six possible cities for {\tt A}. Similarly, rule {\tt (4)} would generate 60 instances as there are 30 pairs {\tt A,B} of different cities and two cars for {\tt X}. Many of these ground rules, however, are irrelevant. Take, for instance:

\begin{verbatim}
o at(1,ny) ::- driveto(1,ny).
no_at(1,paris) ::- at(1,ny).
\end{verbatim}

\noindent corresponding to possible instantiations of {\tt (1)} and {\tt (4)}, respectively. In both cases, the body refers to a situation where car 1 is located or will drive to New York, while we can observe that it was initially at Madrid and that the European roadmap is disconnected from the American one. Of course, one could additionally encode a static reachability predicate to force that rule instances refer to reachable cities for a given car, but this would not be too transparent or elaboration tolerant. 
%One would expect that the grounder was capable of detecting these ``non-derivable'' cases ignoring them in the final ground theory, if possible.

On the other hand, if we forget, for a moment, the temporal operators and we consider the definition of safe variables used in ASP, one may also wonder whether it is possible to simply require that each variable occurs in the positive body of rules, without needing to refer to static predicates mandatorily. Figure~\ref{fig:cars2} contains a possible variation of the same scenario allowing this possibility. Our goal is allowing this new, more flexible definition of safe variables and exploiting, if possible, the information in the temporal program to reduce the set of generated ground rules.

\begin{figure}[htbp]
\begin{center}
\begin{Verbatim}[frame=single]
static city/1, car/1, road/2.

o at(X,A)  ::- driveto(X,A).

driveto(X,B) v no_driveto(X,B) ::- at(X,A), road(A,B).

o at(X,A)    ::- at(X,A), not o no_at(X,A).
  no_at(X,A) ::- at(X,B), A!=B, city(A).

::- at(X,A), at(X,B), A!=B.
\end{Verbatim}
\end{center}
\caption{A possible variation of the cars scenario.}
\label{fig:cars2}
\end{figure}

%%%%%%%%%%%%%%%%%%%%%%%%%%%%%%%%%%%%%%%%%%%%%%%%%%%%%%%%%%%%%%%%%%%%%%%%%%%%%%%%%%
\section{Temporal Quantified Equilibrium Logic}\label{sec:qel}
%%%%%%%%%%%%%%%%%%%%%%%%%%%%%%%%%%%%%%%%%%%%%%%%%%%%%%%%%%%%%%%%%%%%%%%%%%%%%%%%%%

Syntactically, we consider function-free first-order languages $\LDef$ built over a set of \emph{constant} symbols, $C$, and a set of
\emph{predicate} symbols, $P$.
Using $\mathcal{L}$, connectors and variables, an $\LDef$-\emph{formula} $F$ is defined following the grammar:
%\begin{eqnarray*}
\begin{gather*}
F ::= p \ | \ \bot \ | \ F_1 \wedge F_2 \ | \ F_1 \vee F_2 \ | \ F_1 \rightarrow F_2 \ | \ \\  \Next \, F \ |  \  \ \nec F \ | \ \pos F \ | \  \forall x F(x) \ | \ \exists x F(x)
\end{gather*}
%\end{eqnarray*}
\noindent where $p \in P$ is an atom, $x$ is a variable and $\Next$, $\nec$ and $\pos$ respectively stand for ``next'', ``always'' and ``eventually.'' A \emph{theory} is a finite set of formulas. We use the following derived operators:

\[
\begin{array}{rcl}
\neg F & \eqdef & F \rightarrow \bot\\
\top & \eqdef & \neg \bot \\
F \leftrightarrow G & \eqdef & (F \rightarrow G) \wedge (G \rightarrow F)
\end{array}
\]
\noindent for any formulas $F,G$. An \emph{atom} is any $p(t_1, \ldots, t_n)$ where $p \in P$ is a predicate with $n$-arity and each $t_i$ is a term (a constant or a variable) in its turn. We say that a term or a formula is \emph{ground} if it does not contain variables. An $\L$-\emph{sentence} or closed-formula is a formula without free-variables.

The application of $i$ consecutive $\Next$'s is denoted as follows: $\Next^i \varphi \eqdef \Next (\Next^{i-1} \varphi)$ for $i>0$ and $\Next^0 \varphi \eqdef \varphi$.
A \emph{temporal fact} is a construction of the form $\Next^i A$ where $A$ is an atom.

Let $D$ be a non-empty set (the \emph{domain} or \emph{universe}). By $At(D,P)$ we denote the set of ground atomic sentences of the language $\langle D, P\rangle $. We will also define an interpretation $\sigma$ of constants in $C$ (and domain elements in $D$) as a mapping \label{sigma}$$\sigma\colon C \cup D  \rightarrow D$$ such that $\sigma(d)=d$ for all $d\in D$.

A first-order LTL-interpretation is a structure $\tuple{(D,\sigma),\T}$ where $D$ and $\sigma$ are as above and $\T$ is an infinite sequence of sets, $\T=\{T_i\}_{i\geq 0}$ with $T_i \subseteq At(D,P)$. Intuitively, $T_i$ contains those ground atoms that are true at situation $i$.  Given two LTL-interpretations $\H$ and $\T$ we say that $\H$ is \emph{smaller than} $\T$, written $\H \leq \T$, when $H_i \subseteq T_i$ for all $i\geq 0$. As usual, $\H < \T$ stands for: $\H \leq \T$ and $\H \neq \T$. We define the ground temporal facts associated to $\T$ as follows: $Facts(\T)\eqdef \{\Next^i p \ | \ p \in T_i\}$. It is easy to see that $\H \leq \T$ iff $Facts(\H) \subseteq Facts(\T)$.

Given $\T$ as above, we denote by $\T\res C$ the sequence of sets $\{T_i \res C\}_{i \geq 0}$, where each $T_i \res C=T_i \cap At(\sigma(C),P)$, i.e., those atoms from $T_i$ that contain terms  exclusively formed with universe elements that are images of syntactic constants in $C$.

\begin{definition}\label{def:struct}
A \emph{temporal-here-and-there} ${\cal L}$-structure with
static domains, or a \emph{$\mathbf{TQHT}$-structure}, is a
tuple $\M = \langle (D,\sigma), \H,  \T \rangle$ where
$\langle (D,\sigma), \H \rangle$ and $\langle (D,\sigma), \T \rangle$ are two LTL-interpretations satisfying $\H \leq \T$.\qed
\end{definition}

A $\mathbf{TQHT}$-structure of the form $\M = \langle (D,\sigma), \T,  \T \rangle$ is said to be \emph{total}. If $\M = \langle (D,\sigma), \H,  \T \rangle$ is a $\mathbf{TQHT}$-structure and $k$ any positive integer, we denote by $(\M,k)=\langle (D,\sigma),( \H,k),  (\T,k) \rangle$ the temporal-here-and-there ${\cal L}$-structure with $(\H,k)=\{H_i\}_{i \geq k}$ and  $(\T,k)=\{T_i\}_{i \geq k}$. The satisfaction relation for $\M=\langle (D,\sigma), \H,  \T \rangle$ is defined recursively forcing us to consider formulas from $\langle C\cup D, P\rangle $. Formally, if $\varphi $ is an $\mathcal{L}$-sentence for the atoms in $At(C\cup D,P)$, then:

\begin{itemize}
\item If $\varphi =p(t_1,\dots,t_n)\in At(C\cup D,P)$, then \begin{align*}
\M \models p(t_1,\dots,t_n)\quad &\mbox{ iff } \quad
p(\sigma(t_1),\dots,\sigma(t_n))\in H_0.\\
\M \models t=s\;\; & \mbox{ iff } \quad
\sigma(t)=\sigma(s)
\end{align*}
\item
$\M \not\models \bot$

\item
$\M \models \varphi \land \psi$ iff $\M \models
\varphi$ and $\M \models \psi$.
\item
$\M \models \varphi \lor \psi$ iff $\M\models
\varphi$ or $\M \models \psi$.
\item
$\M \models \varphi \to \psi$ iff $\langle (D,\sigma), w,  \T \rangle \not\models \varphi$ or $\langle (D,\sigma), w,  \T \rangle \models \psi$ for all $w \in\{\H,\T\}$
\item $\M \models \Next \varphi$ \ \ if $(\M,1) \models \varphi$.
\item $\M \models \nec \varphi$ \ \ if $\forall j \geq 0$, \ $(\M,j) \models \varphi$
\item $\M \models \pos \varphi$ \ \ if $\exists j \geq 0$, \ $(\M,j) \models \varphi$
%\item
%$\langle (D,\sigma), \T,  \T \rangle \models \forall x \varphi(x)$ iff $\langle (D,\sigma), \T,  \T \rangle \models
%\varphi(d)$ for all  $d\in D$.
\item
$\langle (D,\sigma), \H,  \T \rangle \models \forall x \varphi(x)$ iff  $\langle (D,\sigma), w,  \T \rangle \models
\varphi(d)$ for all  $d\in D$ and for all $w \in\{\H,\T\}$.
\item
$\M  \models \exists x \varphi(x)$ iff
$\M \models\varphi(d)$ for some $d\in D$.
\end{itemize}

The resulting logic is called \emph{Quantified Temporal Here-and-There Logic with static domains}, and denoted by $\mathbf{SQTHT}$ or simply by $\mathbf{QTHT}$. It is not difficult to see that, if we restrict to total $\mathbf{TQHT}$-structures, $\langle (D,\sigma), \T,  \T \rangle \models \varphi$ iff $\langle (D,\sigma), \T,  \T \rangle\models \varphi$ in first-order LTL. Furthermore, the following property can be easily  checked by structural induction.
\begin{proposition} \label{prop:1}
For any formula $\varphi$, if $\tuple{(D,\sigma), \H,\T} \models \varphi$, then:
\[
\tuple{(D,\sigma), \T,\T} \models \varphi
\]
 \end{proposition}

A \emph{theory} $\Gamma$ is a set of $\mathcal{L}$-sentences. An interpretation $\M$ is a model of a theory $\Gamma$, written $\M \models \Gamma$, if it satisfies all the sentences in $\Gamma$.

\begin{definition}[Temporal Equilibrium Model]
A {\em temporal equilibrium} model of a theory $\Gamma$ is a total model $\M=\langle
(D,\sigma),\T,\T\rangle$ of $\Gamma$ such that there is no $\H < \T$ satisfying $\langle (D,\sigma),\H,\T\rangle \models \Gamma$.\qed
\end{definition}

If $\M=\langle(D,\sigma),\T,\T\rangle$ is a temporal equilibrium model of a theory $\Gamma$, we say that the First-Order LTL interpretation $\langle(D,\sigma),\T\rangle$ is a \emph{temporal stable model} of $\Gamma$. We write $TSM(\Gamma)$ to denote the set of temporal stable models of $\Gamma$. The set of \emph{credulous consequences} of a theory $\Gamma$, written $CredFacts(\Gamma)$ contains all the temporal facts that occur at some temporal stable model of $\Gamma$, that is:
\begin{eqnarray*}
CredFacts(\Gamma) & \eqdef & \bigcup_{\tuple{(D,\sigma),\T} \in TSM(\Pi)} Facts(\T)
\end{eqnarray*}

A property of TEL directly inherited from Equilibrium Logic (see Proposition 5 in~\cite{Pea06}) is the following:
\begin{proposition}[Cumulativity for negated formulas]\label{prop:cum}
Let $\Gamma$ be some theory and let $\neg \varphi$ be some formula such that $\M \models \neg \varphi$ for all temporal equilibrium models of $\Gamma$. Then, the theories $\Gamma$ and $\Gamma \cup \{\neg \varphi\}$ have the same set of temporal equilibrium models.\qed
\end{proposition}

%It is well-known that stable models (and so Equilibrium Logic) do not satisfy cummulativity in the general case: that is, if a formula is satisfied in all the stable models, adding it to the program may vary the consequences we obtain. However, when we deal with negated formulas, Proposition~\ref{prop:cum} tells us that cummulativity is guaranteed.

In this work, we will further restrict the study to a syntactic subset called \emph{splittable}
%\footnote{The name \emph{splittable} refers to the fact that these programs can be splitted using~\cite{LT94} thanks to the property that rule heads never refer to a time point previous to those referred in the body.}}
temporal formulas (STF) which will be of one of the following types:

%\begin{itemize}
%\item If $\varphi$ is a quantifier-free formula, then $\varphi$ should be
\begin{eqnarray}
B \wedge N & \rightarrow & H \label{f:in0} \\
B \wedge \Next B' \wedge N \wedge \Next N' & \rightarrow & \Next H' \label{f:in1} \\
\nec(B \wedge \Next B' \wedge N \wedge \Next N' & \rightarrow & \Next H')
\label{f:d}
\end{eqnarray}
\noindent where $B$ and $B'$ are conjunctions of atomic formulas, $N$ and $N'$ are conjunctions of $\neg p$, being $p$ an atomic formula and  $H$ and $H'$ are disjunctions of atomic formulas.

\begin{definition}
A \emph{splittable temporal logic program} (\emph{STL-program} for short) is a finite set of sentences like $$\varphi= \forall x_1 \forall x_2 \ldots \forall x_n \psi \label{f:q},$$
\noindent where $\psi$ is a splittable temporal formula with $x_1,x_2,\ldots,x_n$ free variables.
\end{definition}

We will also accept in an STL-program an implication of the form $\nec (B \wedge N \rightarrow H)$ (that is, containing $\nec$ but not any $\Next$) understood as an abbreviation of the pair of STL-formulas:
\begin{eqnarray*}
B \wedge N & \rightarrow & H \\
\nec(\Next B \wedge \Next N & \rightarrow & \Next H)
\end{eqnarray*}

\begin{example}\label{ex1}
The following theory $\Pi_{\ref{ex1}}$ is an STL-program:
\begin{eqnarray}
\neg p & \rightarrow & q \label{ex1.1} \\
 q \wedge \neg \Next r & \rightarrow & \Next p \label{ex1.2} \\
\nec ( q \wedge \neg \Next p & \rightarrow & \Next q ) \label{ex1.3} \\
\nec ( r \wedge \neg \Next p & \rightarrow & \Next r \vee \Next q) \label{ex1.4}
\end{eqnarray}
\end{example}

For an example including variables, the encoding of Example~\ref{ex:cars} in Figure~\ref{fig:cars2} is also an STL-program $\Pi_{\ref{ex:cars}}$ whose logical representation corresponds to:
\begin{eqnarray}
\nec (\ Driveto(x,a) & \rightarrow & \Next At(x,a) \ ) \label{f:car.1}\\
\nec (\ At(x,a) \wedge Road(a,b) & \rightarrow & Driveto(x,b) \vee NoDriveto(x,b) \ )  \label{f:car.2}\\
\nec (\ At(x,a) \wedge \neg \Next NoAt(x,a) & \rightarrow & \Next At(x,a) \ )  \label{f:car.3}\\
\nec (\ At(x,b) \wedge City(a) \wedge a \neq b & \rightarrow & NoAt(x,a) \ )  \label{f:car.4}\\
\nec (\ At(x,a) \wedge At(x,b) \wedge a \neq b & \rightarrow & \bot \ )  \label{f:car.5}
\end{eqnarray}

\noindent Remember that all rule variables are implicitly universally quantified. For simplicity, we assume that inequality is a predefined predicate.

An STL-program is said to be \emph{positive} if for all rules \eqref{f:in0}-\eqref{f:d}, $N$ and $N'$ are empty (an empty conjunction is equivalent to $\top$). An STL-program is said to be \emph{normal} if it contains no disjunctions, i.e., for all rules \eqref{f:in0}-\eqref{f:d}, $H$ and $H'$ are atoms.

Given a propositional combination $\varphi$ of temporal facts with $\wedge, \vee, \bot, \rightarrow$, we denote $\varphi^i$ as the formula resulting from replacing each temporal fact $A$ in $\varphi$ by $\Next^i A$. For a formula $r= \nec \varphi$ like \eqref{f:d}, we denote by $r^i$ the corresponding $\varphi^i$. For instance, $\eqref{ex1.3}^i=(\Next^i q \wedge \neg \Next^{i+1} p \rightarrow \Next^{i+1} q)$. As $\Next$ behaves as a linear operator in THT, in fact $F^i \leftrightarrow \Next^i F$ is a THT tautology.

\begin{definition}[expanded program]
Given an STL-program $\Pi$ for signature $At$ we define its \emph{expanded program} $\Pi^\infty$ as the infinitary logic program containing all rules of the form \eqref{f:in0}, \eqref{f:in1} in $\Pi$ plus a rule $r^i$ per each rule $r$ of the form \eqref{f:d} in $\Pi$ and each integer value $i\geq 0$. \qed
\end{definition}

The program $\Pi^\infty_{\ref{ex1}}$ consists of \eqref{ex1.1}, \eqref{ex1.2} plus the infinite set of rules:
\begin{eqnarray*}
\Next^i q \wedge \neg \Next^{i+1} p & \rightarrow & \Next^{i+1} q \\
\Next^i r \wedge \neg \Next^{i+1} p & \rightarrow & \Next^{i+1} r \vee \Next^{i+1} q
\end{eqnarray*}
\noindent for $i \geq 0$. We can interpret the expanded program as an infinite, non-temporal program where the signature is the infinite set of atoms $\{\Next^i p \mid p \in At, \ i \geq 0\}$.

\begin{theorem}[Theorem 1 in~\cite{ACPV11}]\label{th:expand}
$\tuple{\T,\T}$ is a temporal equilibrium model of $\Pi$ iff $\{\Next^i p \ | \ p \in T_i, i \geq 0\}$ is a stable model of $\Pi^\infty$ under the (infinite) signature $\{\Next^i p \ | \ p \in At\}$.\qed
\end{theorem}

\begin{proposition}\label{prop:unique}
Any normal positive STL-program $\Pi$ has a unique temporal stable model $\tuple{(D,\sigma),\T}$ which coincides with its $\leq$-least LTL-model. We denote $LM(\Pi)=Facts(\T)$. \qed
\end{proposition}

%%%%%%%%%%%%%%%%%%%%%%%%%%%%%%%%%%%%%%%%%%%%%%%%%%%%%%%%%%%%%%%%%%%%%%%%%%%%%%%%%%
\section{Safe Variables and Domain Independence} \label{sec:safe}
%%%%%%%%%%%%%%%%%%%%%%%%%%%%%%%%%%%%%%%%%%%%%%%%%%%%%%%%%%%%%%%%%%%%%%%%%%%%%%%%%%

In this section we consider a definition of safe variables for temporal programs that removes the reference to static predicates.

%which does not refer to static predicates any more. As a result, we obtain a direct extrapolation of ASP-safety by just ignoring the temporal operators.
\begin{definition}
\label{def:safe}
A splittable temporal formula $\varphi$ of type (\ref{f:in0}), (\ref{f:in1}) or (\ref{f:d}) is said to be \emph{safe} if, for any variable $x$ occurring in $\varphi$, there exists an atomic formula $p$ in $B$ or $ B'$ such that $x$ occurs in $p$. A formula $\forall x_1 \forall x_2 \ldots \forall x_n \psi$ is safe if the splittable temporal formula $\psi$ is safe.
\end{definition}

\noindent For instance, rules \eqref{f:car.1}-\eqref{f:car.5} are safe. A simple example of an unsafe rule is the splittable temporal formula:
\begin{eqnarray}
\top \rightarrow p(x) \label{f:unsafe}
\end{eqnarray}
\noindent where $x$ does not occur in the positive body. Although an unsafe rule does not always lead to lack of domain independence (see examples in~\cite{CPV09}) it is frequently the case. 
%For instance, if a program contains \eqref{f:unsafe} and we add a new fresh constant $c$ to the signature, stable models will contain the ground fact $P(c)$, something that was obviously impossible before the addition of constant $c$ not mentioned before. 
We prove next that domain independence is, in fact, guaranteed for safe STL-programs.

%\begin{lemma}
%\label{safe_quantifier_free}
%Let $\varphi$ be a splittable temporal formula and $\mu$ a variable assignment in $(D, \sigma )$. If $\varphi$ is safe, then if follows that:
%$$\langle (D,\sigma),\T,\T\rangle\models \varphi^\mu \mbox { implies } \langle (D,\sigma),\T\res C,\T\rangle\models \varphi^\mu.$$
%\end{lemma}
%

%\begin{proposition}
%\label{safe_formula}
%For any safe sentence $\varphi=\forall x_1 \forall x_2 \ldots \forall x_n \psi$
%\begin{eqnarray*}
%\langle (D,\sigma),\T,\T\rangle\models \varphi & \hbox{iff} & \langle (D,\sigma),\T\res C,\T\rangle\models \varphi.
%\end{eqnarray*}
%\end{proposition}
%
%\noindent

%
\begin{theorem}
\label{th:safe}
If $\varphi$ is a safe sentence and
$\langle (D,\sigma),\T,\T\rangle$ is a temporal equilibrium model of $\varphi$, then
$\T\res C=\T$ and $T_i \subseteq At(\sigma(C), P)$ for any $i \geq 0$.
\end{theorem}

Let $(D,\sigma)$ be a domain and $D'\subseteq D$ a finite subset; the grounding over $D'$ of a sentence $\varphi$, denoted by $\mathrm{Gr}_{D'}(\varphi)$, is defined recursively

\begin{eqnarray*}
\mathrm{Gr}_{D'}(p) & \eqdef & p \textrm{, where $p$ denotes any atomic formula}\\
\mathrm{Gr}_{D'}(\varphi_1\odot\varphi_2) & \eqdef & \mathrm{Gr}_{D'}(\varphi_1) \odot \mathrm{Gr}_{D'}(\varphi_2), \\
& & \textrm{with $\odot$  any binary operator in  $\{\wedge, \vee, \rightarrow\}$}\\
\mathrm{Gr}_{D'}(\forall x\varphi(x)) & \eqdef & \textstyle\bigwedge\limits_{d\in D'}\mathrm{Gr}_{D'}\varphi(d) \\
\mathrm{Gr}_{D'}(\exists x\varphi(x)) & \eqdef & \textstyle\bigvee\limits_{d\in D'}\mathrm{Gr}_{D'}\varphi(d) \\
\mathrm{Gr}_{D'}(\Next \varphi) & \eqdef & \textstyle\Next\mathrm{Gr}_{D'}(\varphi) \\
\mathrm{Gr}_{D'}(\nec \varphi) & \eqdef & \textstyle\nec \mathrm{Gr}_{D'}(\varphi) \\
\mathrm{Gr}_{D'}(\pos \varphi) & \eqdef & \textstyle\pos\mathrm{Gr}_{D'}(\varphi)
\end{eqnarray*}
\noindent

\begin{theorem}[Domain independence]
\label{th:safe_three}
Let $\varphi$ be  safe splittable temporal sentence. Suppose we expand
the language $\mathcal{L}$ by considering a set of constants $C' \supseteq
C$. A total $\mathbf{QTHT}$-model $\langle (D,\sigma),\T,\T\rangle$ is a
temporal equilibrium model of $ \mathrm{Gr}_{C'}(\varphi)$ if and only if
it is a temporal equilibrium model of $ \mathrm{Gr}_{C}(\varphi)$.
\end{theorem}

%%%%%%%%%%%%%%%%%%%%%%%%%%%%%%%%%%%%%%%%%%%%%%%%%%%%%%%%%%%%%%%%%%%%%%
\section{Derivable ground facts} \label{sec:ground}
%%%%%%%%%%%%%%%%%%%%%%%%%%%%%%%%%%%%%%%%%%%%%%%%%%%%%%%%%%%%%%%%%%%%%%

In this section we present a technique for grounding safe temporal programs based on the construction of a positive normal ASP program with variables. 
%The least model of this program can be obtained by an ASP grounder and it can be used afterwards to provide the variable substitutions to be performed on the STL-program. Besides, in some cases, this technique means a reduction of the number of generated ground rules with respect to the previous strategy that relied on static predicates.
The method is based on the idea of \emph{derivable} ground temporal facts for an STL-program $\Pi$. This set, call it $\Delta$, will be an upper estimation of the credulous consequences of the program, that is, $CredFacts(\Pi) \subseteq \Delta$. Of course, the ideal situation would be that $\Delta = CredFacts(\Pi)$, but the set $CredFacts(\Pi)$ requires the temporal stable models of $\Pi$ and these (apart from being infinite sequences) will not be available at grounding time. In the worst case, we could choose $\Delta$ to contain the whole set of possible temporal facts, but this would not provide relevant information to improve grounding. So, we will try to obtain some superset of $CredFacts(\Pi)$ as small as possible, or if preferred, to obtain the largest set of \emph{non-derivable} facts we can find. Note that a non-derivable fact $\Next^i p \not\in \Delta$ satisfies that $\Next^i p \not\in CredFacts(\Pi)$ and so, by Proposition~\ref{prop:cum}, $\Pi \cup \{\neg \! \Next^i p\}$ is equivalent to $\Pi$, that is, both theories have the same set of temporal equilibrium models. This information can be used to simplify the ground program either by removing rules or literals.

We begin defining several transformations on STL-programs. For any temporal rule $r$, we define $r^\wedge$ as the set of rules:
\begin{itemize}
\item If $r$ has the form \eqref{f:in0} then
$r^\wedge \eqdef \{B \rightarrow p \ | \ \hbox{atom } p \ \hbox{occurs in } H\}$
\item If $r$ has the form \eqref{f:in1} then
$r^\wedge \eqdef \{B \wedge \Next B' \rightarrow \Next p \ | \ \hbox{atom } p \ \hbox{occurs in } H'\}$
\item If $r$ has the form \eqref{f:d} then
$r^\wedge \eqdef \{\nec(B \wedge \Next B' \rightarrow \Next p) \ | \ \hbox{atom } p \ \hbox{occurs in } H'\}$
\end{itemize}
In other words, 
%$r^\wedge$ results from removing all negative literals in $r$ and, informally speaking, transforming disjunctions in the head into conjunctions, so that 
$r^\wedge$ will imply \emph{all} the original disjuncts in the disjunctive head of $r$. It is interesting to note that for any rule $r$ with an empty head ($\bot$) this definition implies $r^\wedge = \emptyset$. Program $\Pi^\wedge$ is defined as the union of $r^\wedge$ for all rules $r \in \Pi$. As an example, $\Pi_{\ref{ex1}}^\wedge$ consists of the rules:
\[
\begin{array}{c@{\hspace{20pt}}c@{\hspace{20pt}}c}
\begin{array}{rcl}
\top & \rightarrow & q \\
 q  & \rightarrow & \Next p \\
\end{array}
&
\begin{array}{rcl}
\nec ( q & \rightarrow & \Next q )\\ \\
\end{array}
&
\begin{array}{rcl}
\nec ( r & \rightarrow & \Next r ) \\
\nec ( r & \rightarrow & \Next q )
\end{array}
\end{array}
\]
\noindent whereas $\Pi^\wedge_{\ref{ex:cars}}$ would be the program:
\begin{eqnarray}
\nec (\ Driveto(x,a) & \rightarrow & \Next At(x,a) \ ) \label{f:carb.1}\\
\nec (\ At(x,a) \wedge Road(a,b) & \rightarrow & Driveto(x,b) \ )  \label{f:carb.2}\\
\nec (\ At(x,a) \wedge Road(a,b) & \rightarrow & NoDriveto(x,b) \ )  \label{f:carb.2b}\\
\nec (\ At(x,a) & \rightarrow & \Next At(x,a) \ )  \label{f:carb.3}\\
\nec (\ At(x,b) \wedge City(a) \wedge a \neq b & \rightarrow & NoAt(x,a) \ )  \label{f:carb.4}
\end{eqnarray}

%\noindent If we look carefully at this example program, we are now moving each car $x$ so that it will be at several cities at the same time (constraint \eqref{f:car.5} has been removed) and, at each step, it will additionally locate car $x$ in all adjacent cities to the previous ones ``visited'' by $x$. In this way, if we conclude $\Next^i At(x,a)$ from this program this is actually representing that car $x$ \emph{can reach} city $a$ in $i$ steps or less. In some sense, $\Pi^\wedge$ looks like a heuristic simplification\footnote{We could further simplify $\Pi^\wedge$ removing rules \eqref{f:carb.2b} and \eqref{f:carb.4} by observing that their head predicates never occur in a positive body of $\Pi_{\ref{ex:cars}}$. However, for the formal results in the paper, this is not essential, and would complicate the definitions.} of the original problem obtained by removing some constraints.
% (this is something common in the area of Planning in Artificial Intelligence).

Notice that, by definition, $\Pi^\wedge$ is always a positive normal STL-program and, by Proposition~\ref{prop:unique}, it has a unique temporal stable model, $LM(\Pi^\wedge)$.

\begin{proposition}\label{prop:cred_one}
For any STL-program $\Pi$, $CredFacts(\Pi) \subseteq LM(\Pi^\wedge)$.\qed
\end{proposition}

Unfortunately, using $\Delta=LM(\Pi^\wedge)$ as set of derivable facts is unfeasible, since it contains infinite temporal facts corresponding to an ``infinite run'' of the transition system described by $\Pi^\wedge$. 
%Take for instance $\Pi^\wedge_{\ref{ex:cars}}$ for the cars scenario. Imagine a roadmap with thousands of connected cities. $LM(\Pi^\wedge)$ can tell us that, for instance, car $1$ cannot reach Berlin in less than $316$ steps, so that $\Next^{315} At(1,Berlin)$ is non-derivable, although $\Next^{316} At(1,Berlin)$ is derivable. However, in order to exploit this information for grounding, we would be forced to expand the program up to some temporal distance, and we have no hint on where to stop. 
%On the other hand, when we represent the transition system as usual in ASP (i.e. with a bounded integer variable for time), this fine-grained optimization for grounding can be applied, because the temporal path \emph{always has a finite length}.
Instead, we will adopt a compromise solution taking a superset of $LM(\Pi^\wedge)$ extracted from a new theory, $\Gamma_\Pi$. This theory will collapse all the temporal facts from situation $2$ on, so that all the states $T_i$ for $i\geq 2$ will be repeated\footnote{\label{footn} Note that rules of the form \eqref{f:in0} and \eqref{f:in1} are not in the scope of $\nec$ and so may provide an irregular behaviour for atoms at situations $0$ and $1$. In a theory only consisting of rules like~\eqref{f:d} we could collapse all situations from $i=0$ on since they would follow a regular pattern.}. We define $\Gamma_\Pi$ as the result of replacing each rule $\nec(B \wedge \Next B' \rightarrow \Next p)$ in $\Pi^\wedge$ by the formulas:
\begin{eqnarray}
 B \wedge \Next B' & \rightarrow & \Next p \label{f:copy1} \\
\Next B \wedge \Next^2 B' & \rightarrow & \Next^2 p \label{f:copy2} \\
\Next^2 B \wedge \Next^2 B' & \rightarrow & \Next^2 p \label{f:copy3}
\end{eqnarray}
\noindent and adding the axiom schema:
\begin{eqnarray}
\Next^2 \nec (p \leftrightarrow \Next p) \label{f:repeat}
\end{eqnarray}
\noindent for any ground atom $p \in At(D,P)$ in the signature of $\Pi$. As we can see, \eqref{f:copy1} and \eqref{f:copy2} are the first two instances of the original rule $\nec(B \wedge \Next B' \rightarrow \Next p)$ corresponding to situations $i=0$ and $i=1$. Formula \eqref{f:copy3}, however, differs from the instance we would get for $i=2$ since, rather than having $\Next^3 B'$ and $\Next^3 p$, we use $\Next^2 B'$ and $\Next^2 p$ respectively. This can be done because axiom \eqref{f:repeat} is asserting that from situation 2 on all the states are repeated.

In the cars example, for instance, \eqref{f:carb.1} from $\Pi^\wedge_{\ref{ex:cars}}$ would yield the three rules:
\begin{eqnarray*}
Driveto(x,a) & \rightarrow & \Next At(x,a) \\
\Next Driveto(x,a) & \rightarrow & \Next^2 At(x,a) \\
\Next^2 Driveto(x,a) & \rightarrow & \Next^2 At(x,a)
\end{eqnarray*}
It is not difficult to see that axiom \eqref{f:repeat} implies that checking that some $\M$ is a temporal equilibrium model of $\Gamma_\Pi$ is equivalent to checking that $\{\Next^i p \ | \ p \in T_i \, , \, i=0,1,2\}$ is a stable model of $\Gamma_\Pi \setminus \{ \eqref{f:repeat}\}$ and fixing $T_i=T_2$ for $i \geq 3$. This allows us to exclusively focus on the predicate extents in $T_0, T_1$ and $T_2$, so we can see the $\nec$-free program $\Gamma_\Pi \setminus \{ \eqref{f:repeat}\}$ as a positive normal ASP (i.e., non-temporal) program for the propositional signature $\{p, \Next p, \Next^2 p \ | \ p \in At(D,P)\}$ that can be directly fed to an ASP grounder, after some simple renaming conventions.

\begin{theorem}\label{th:delta}
$\Gamma_\Pi$ has a least LTL-model, $LM(\Gamma_\Pi)$ which is a superset of $LM(\Pi^\wedge)$.
\end{theorem}

In other words $CredFacts(\Pi) \subseteq LM(\Pi^\wedge) \subseteq LM(\Gamma_\Pi) = \Delta$, i.e., we can use $LM(\Gamma_\Pi)$ as set of derivable facts and simplify the ground program accordingly. %Note that this simplification does not mean that we first ground everything and then remove rules and literals: we simply do not generate the irrelevant ground cases.
To this aim, a slight adaptation is further required. Each rule in $\Pi$ like~\eqref{f:d} has the form $\nec \alpha$ and any predicate $p$ in $\alpha$ is \emph{implicitly} affected (Theorem~\ref{th:expand}) by the extension of $\Next^2 p$ in $LM(\Gamma_\Pi)$. In order to properly ground the extensions for $p, \Next p$ and $\Next^2 p$ we replace each $\nec \alpha$ by the equivalent conjunction of the three rules $\alpha$, $\Next \alpha$ and $\nec \Next^2 \alpha$. For instance, \eqref{f:car.2} would be replaced by:
\begin{eqnarray}
 At(x,a) \wedge Road(a,b) & \rightarrow & Driveto(x,b) \vee NoDriveto(x,b) \hspace{15pt} \\
\Next At(x,a) \wedge \Next Road(a,b) & \rightarrow & \Next Driveto(x,b) \nonumber \\
& & \vee \Next NoDriveto(x,b) \\
\nec (\ \Next^2 At(x,a) \wedge \Next^2 Road(a,b) & \rightarrow & \Next^2 Driveto(x,b) \nonumber \\
 & & \vee \Next^2 NoDriveto(x,b) \ ) \label{f:1}
\end{eqnarray}
\noindent and then check the possible extents for the positive bodies we get from the set of derivable facts $\Delta=LM(\Gamma_\Pi)$. For example, for the last rule, we can make substitutions for $x, a$ and $b$ using the extents of $\Next^2 At(x,a)$ and $\Next^2 Road(a,b)$ we have in $\Delta$. However, this still means making a join operation for both predicates. We can also use the ASP grounder for that purpose by just adding a rule that has as body, the positive body of the original temporal rule $r$, and as head, a new auxiliary predicate $Subst_r(x,a,b)$ referring to all variables in the rule. In the example, for rule \eqref{f:1} we would include in our ASP program:
\begin{eqnarray*}
\Next^2 At(x,a) \wedge \Next^2 Road(a,b) \rightarrow Subst_{\eqref{f:1}}(x,a,b)
\end{eqnarray*}

In this way, each tuple of $Subst_r(x_1,\dots,x_n)$ directly points out the variable substitution to be performed on the temporal rule. 

For instance, in the small instance case described of our example ($2$ cars and $6$ cities) we reduce the number of generated ground rules in the scope of `$\nec$' from $160$ using the previous {\tt STeLP} grounding method to $62$. The reader may easily imagine that the higher degree of cities interconnection, the smaller obtained reduction of rule instances. Although an exhaustive experimentation is still ongoing work, a reduction of this kind is very promising. In our initial experiments, the grounding performed on $\Gamma_\Pi$ (whose generation is polynomial) does not constitute a significant time increase, whereas the computation of temporal stable models is drastically improved by the reduction of ground rules\footnote{The complexity of deciding whether a temporal stable model exists is {\textsc EXPSPACE}-complete in the general case~\cite{BP15}.}.

%, as this approaches to the worst case of $n^2$, where the $n$ cities are all pairwise connected. 
%On the other hand, the example is general enough to illustrate the proposed technique, as rules for temporal predicates usually limit the possible combinations of variable values we must consider.

%%%%%%%%%%%%%%%%%%%%%%%%%%%%%%%%%%%%%%%%%%%%%%%%%%%%%%%%%%%%%%%%%%%%%%
\section{Conclusions} \label{sec:conc}
%%%%%%%%%%%%%%%%%%%%%%%%%%%%%%%%%%%%%%%%%%%%%%%%%%%%%%%%%%%%%%%%%%%%%%

We have improved the grounding method for temporal logic programs with variables in different ways. First, we provided a safety condition that directly corresponds to extrapolating the usual concept of safe variable in ASP. In this way, any variable occurring in a rule is considered to be safe if it also occurs in the positive body of the rule, regardless the possible scope of temporal operators and removing the previous dependence on the use of static predicates.

We have proved that this safety condition suffices to guarantee the property of \emph{domain independence} by which computing the (temporal) stable models is insensitive to the possible addition of new arbitrary constants to the universe. 
%The current tool for computing temporal stable models, called {\tt STeLP}, handled a safety condition which was strictly stronger than the one presented in this note, requiring the further use of static predicates that acted as variable types.

%Some recent papers have studied safety conditions in ASP for syntatic classes that are more general than disjunctive logic programs: this is the case of~\cite{BFL08} for so-called \emph{Normal-Form Nested} (NFN) logic programs, or~\cite{LLP08,CPV09} for arbitrary first order theories under the quantified equilibrium logic interpretation (or its equivalent formalisation of General Stable Models). In particular, \cite{CPV09} is more general than~\cite{LLP08} for general theories, while it collapses to~\cite{BFL08} for NFN programs. An interesting topic for future study is trying to extend~\cite{CPV09} to provide a general safety condition for arbitrary quantified temporal theories.

We have also designed a method for grounding the temporal logic program that consists in constructing a non-temporal normal positive program with variables that is fed to an ASP solver to directly obtain the set of variable substitutions to be performed for each rule. The proposed method allows reducing in many cases the number of ground temporal rules generated as a result. 

The current note contains formal results, providing the correctness (with respect to domain independence) of the safety condition and the method for grounding safe programs. Regarding implementation, a stand-alone prototype for proving examples like the one in the paper has been constructed, showing promising results. The immediate next step is incorporating the new grounding method inside {\tt STeLP} and analysing its performance on benchmark scenarios. 
%We will also study different improvements like, for instance, detecting rules with variables that are irrelevant for grounding.

\paragraph{Acknowledgements}{This research was partially supported by Spanish MINECO project TIN2013-42149-P and Xunta de Galicia GPC 2013/070.}

\bibliographystyle{acmtrans}
\bibliography{refs}

\newpage
\section*{Appendix. Proofs}

\begin{proofof}{Proposition~\ref{prop:unique}}
Consider the expanded program $\Pi^\infty$. This is an infinitary positive (non-temporal) logic program. Using the well-known results by~\cite{vEm76}, we know it has a least Herbrand model $LM(\Pi^\infty)$ call it $I$, which may contain an infinite set of atoms in the signature of ground temporal facts $\{ \Next^i p \ | \ p \in At, i\geq 0 \}$, where $At$ is the original signature of $\Pi$. Furthermore, as $\Pi^\infty$ is positive, its unique stable model is precisely $I$. Given any set of ground temporal facts $I$ we can establish a one-to-one correspondence to an LTL-interpretation $\mathbf{I}$ with $Facts(\mathbf{I})=I$. By Theorem~\ref{th:expand}, $J$ is a stable model of $\Pi^\infty$ iff $\mathbf{J}$, with $Facts(\mathbf{J})=J$ is a temporal stable model of $\Pi$. Finally, as $I$ is the unique stable model of $\Pi^\infty$ we get that $\mathbf{I}$ is the only temporal stable model of $\Pi$.
\end{proofof}

\vspace{10pt}
A variable assignment %\footnote{Esto se definir\'{\i}a antes si se habla de ello en la sem\'antica.}%
$\mu$ in $(D,\sigma)$ is a mapping from the set of variables to $D$. %\footnote{Here, an assignment, Would it be static? The assignment of one variable x, would it be the same along the time? }.
If $\varphi\in\Lang$ has free-variables, $\varphi^\mu$ is the closed formula obtained by replacing every free variable $x$ by
$\mu(x)$.

\begin{lemma}
\label{safe_quantifier_free}
Let $\varphi$ be a splittable temporal formula and $\mu$ a variable assignment in $(D, \sigma )$. If $\varphi$ is safe, then if follows that:
$$\langle (D,\sigma),\T,\T\rangle\models \varphi^\mu \mbox { implies } \langle (D,\sigma),\T\res C,\T\rangle\models \varphi^\mu.$$
\end{lemma}

%\begin{proof}[Proof of Lemma~\ref{safe_quantifier_free}]
\begin{proof}
First of all, take $\varphi=B \wedge N \, \rightarrow \, H$ of type \eqref{f:in0} and suppose that $\tuple{\T,\T} \models \varphi^\mu$ but $\tuple{\T \res C, \T} \not \models \varphi^\mu$. This means that $\tuple{\T \res C, \T}  \models B^\mu \wedge N^\mu$ and $\tuple{\T \res C, \T} \not \models H^\mu$. Since $\tuple{\T, \T}  \models H^\mu$, there exists an atomic formula $q$ in $H$ such that
$\tuple{\T , \T}  \models q^\mu$ but $\tuple{\T \res C, \T} \not \models q^\mu$. So we have a variable $x$ in $q$ with $\mu(x) \not \in \sigma(C)$. As $\varphi$ is safe, we know that $x$ occurs in an atomic formula $p$ in $B$. Then $\tuple{\T \res C, \T} \not \models p^\mu$ and $\tuple{\T \res C, \T} \not \models B^\mu$ which yields a contradiction.

If $\varphi$ is of type (\ref{f:in1}), we use a similar argument.

Finally, take $\varphi=\nec(B \wedge \Next B' \wedge N \wedge \Next N' \, \rightarrow \, \Next H')=\nec \psi$ of type (\ref{f:d}) and suppose that $\tuple{\T,\T} \models \varphi^\mu$ but $\tuple{\T \res C, \T} \not \models \varphi^\mu$. There exists $i \geq 0$ such that $\tuple{T_i,T_i} \models \psi^\mu$ and $\tuple{T_i \cap \sigma(C),T_i} \not \models \psi^\mu$. We then have that $\tuple{T_i \cap \sigma(C),T_i}  \models B^\mu \wedge (\Next B')^\mu \wedge N^\mu \wedge (\Next N')^\mu$ and $\tuple{T_i \cap \sigma(C),T_i} \not \models (\Next H')^\mu$. Using that $\varphi$ is safe and the same argument as above, we find an atomic formula $p$ in $B$ or $B'$ such that $\tuple{T_i \cap \sigma(C),T_i} \not \models p^\mu$ which implies $\tuple{T_i \cap \sigma(C),T_i} \not \models B^\mu \wedge (\Next B')^\mu$ and leads to contradiction. The other implication follows directly from Proposition~\ref{prop:1}.
\end{proof}

\begin{proposition}
\label{safe_formula}
For any safe sentence $\varphi=\forall x_1 \forall x_2 \ldots \forall x_n \psi$
\begin{eqnarray*}
\langle (D,\sigma),\T,\T\rangle\models \varphi & \hbox{iff} & \langle (D,\sigma),\T\res C,\T\rangle\models \varphi.
\end{eqnarray*}
\end{proposition}

%\begin{proof}[Proof of Proposition~\ref{safe_formula}]
\begin{proof}
Proceed by induction over the length of the prefix. If $n=0$, we can take any $\mu$ assignment of variables and apply Lemma \ref{safe_quantifier_free} on
$\varphi=\varphi^{\mu}$ . So take $\varphi=\forall x_1 \ldots \forall x_n \psi$ of length $n$ and suppose that the result is true for any universal safe sentence whose prefix has length at most $n-1$. If $\langle (D,\sigma),\T,\T\rangle\models \varphi$, put $\varphi=\forall x_1 \alpha(x_1)$ with $\alpha(x_1)=\forall x_2 \ldots \forall x_n \psi$. For any $d\in D$, we know that  $\langle (D,\sigma),\T,\T\rangle\models \alpha(d)$ and we have to show that $\langle (D,\sigma),\T \res C ,\T\rangle\models \alpha(d)$. The induction hypothesis and the fact that $\alpha(d)$ is a safe sentence whose prefix has length smaller or equal than $n-1$ finishes the proof.
\end{proof}

\begin{proofof}{Theorem~\ref{th:safe}}
If $\varphi$ is a safe sentence and
$\langle (D,\sigma),\T,\T\rangle$ is a temporal equilibrium model of $\varphi$, we have that $\langle (D,\sigma),\T \res C, \T\rangle \models \varphi$ by Proposition~\ref{safe_formula}. The definition of temporal equilibrium model implies that $\T \res C =\T$ and $ T_i \subseteq At(\sigma(C), P)$ for any $i \geq 0$.
\end{proofof}

\begin{lemma}\label{safe&ground&one}
Let $\varphi(x)$ be a safe splittable temporal formula of type (\ref{f:in0}), (\ref{f:in1}) or (\ref{f:d}) and take
$\langle (D,\sigma),\H,\T\rangle$ be such that $\T=\T \res C$. Then, for any $d \in D \setminus \sigma(C)$ we have:
$$ {\langle (D,\sigma),\H,\T\rangle\models \varphi(d)}.$$
\end{lemma}
%\begin{proof}[Proof of Lemma~\ref{safe&ground&one}]
\begin{proof}
First of all, suppose that $\varphi(x)$ is of type (\ref{f:in0}):
$$B \, \wedge N \, \rightarrow \, H$$  and take $d \in D \setminus \sigma(C)$ and $w \in \{ \H,\T\}$ such that ${\langle (D,\sigma),w,\T\rangle\not \models \varphi(d)}$. This implies that ${\langle (D,\sigma),w,\T\rangle \models B(d) \wedge N(d)}$ but ${\langle (D,\sigma),w,\T\rangle\not \models H(d)}$. $\varphi(x)$ is safe so there must be an atom $p$ in $B$ such that $x$ has an occurrence in $p$. Since $T_0\subseteq \mathrm{At}(\sigma(C),P)$, it is clear that ${\langle (D,\sigma),w,\T\rangle\not \models p(d)}$, so ${\langle (D,\sigma),w,\T\rangle\not \models B(d)}$ which yields a contradiction.

\par The proof for the case of $\varphi(x)$ being of type (2) and (3) is similar.
\end{proof}
%%%%%%%%%%%%%%%%%%%%%%%%%%%%%%%%%%%%%%%%%%%%%%%%%%%%%%%%%%%%%%%%%%%%%%%%%%%%%%%%%%%%%%%%%%%
\begin{lemma}
\label{safe&ground&two}
Let $\varphi(x)=\forall x_1 \forall x_2 \ldots \forall x_n \psi$ with $\psi$ a splittable temporal formula and such that $\varphi(x)$ has no other free variables than $x$. Let $\M=\langle (D,\sigma),\H,\T\rangle$ be such that $\T=\T \res C$. Then,  if $\forall x \varphi(x)$ is safe, we have that:
$$ {\M \models \forall x \varphi(x)} \mbox{ iff } {\M \models \textstyle \bigwedge_{c \in C} \varphi(c)}.$$
\end{lemma}

%\begin{proof}[Proof of Lemma~\ref{safe&ground&two}]
\begin{proof}
From left to right, just note that if $ {\M \models \forall x \varphi(x)}$ but ${\M \not \models \varphi(c)}$, for some $c \in C$, we would have that $ {\M \not \models \varphi(\sigma(c))}$ which would yield a contradiction.

For right to left, we can proceed by induction in $n$. If $n=0$, then $\varphi(x)$ is in the case of the previous lemma for any $d \in D \setminus \sigma(C)$, so $ {\M \models \forall x \varphi(x)}$ whenever ${\M \models \textstyle \bigwedge_{c \in C} \varphi(c)}.$ Now, suppose the result is true for any prenex formula with length up to $n-1$ and take $\varphi(x)=\forall x_1 \forall x_2 \ldots \forall x_n \psi(x,x_1,\ldots,x_n)$ such that ${\M \models \textstyle \bigwedge_{c \in C} \varphi(c)}.$ It only rests to show that ${\M \models  \varphi(d)}$ for any $d \in D \setminus \sigma(C)$.
Notice that $\varphi(d)=\forall x_1 \alpha(x_1)$ with $\alpha(x_1)=\forall x_2 \ldots \forall x_n \psi(d,x_1, x_2,\ldots,x_n).$ Since we can apply the induction hypothesis on $\alpha(x_1)$, it will be sufficient to prove that:

$${\M \models \textstyle \bigwedge_{c \in C} \alpha(c)}.$$

\noindent Now fix any $c \in C$ and take into account that $${\M \models  \varphi(c')=\forall x_1 \forall x_2 \ldots \forall x_n \psi(c',x_1, x_2 \ldots,x_n)}$$ for all $c' \in C$, so we can replace $x_1$ by any constant in $C$, including $c$, and so:
$${\M \models  \forall x_2 \ldots \forall x_n \psi(c',c,x_2,\ldots,x_n)}, \mbox{ for any } c' \in C$$
Observe that we can apply the induction hypothesis on $\beta(z)$, where
$$\beta(z)=\forall x_2 \ldots \forall x_n \psi(z,c, x_2, \ldots,x_n)$$
\noindent and then $ {\M \models \forall z \beta(z)}$. In particular ${\M \models \beta(d)}$ which completes the proof since $\beta(d)=\alpha(c)$.
\end{proof}

The following proposition can be easily proved
\begin{proposition}
\label{global_grounding} Given any $D\neq \emptyset$: \ \ $\langle (D, \sigma),\H,\T\rangle\models \varphi \mbox{ iff  } \langle (D, \sigma) ,\H,\T\rangle\models \mathrm{Gr}_{D}(\varphi).$
\qed
\end{proposition}

\begin{theorem}
\label{grounding}
If $\varphi=\forall x_1 \forall x_2 \ldots \forall x_n \psi$ is a safe
splittable temporal sentence and $\M=\langle(D,\sigma),\H,\T\rangle$ such that $\T=\T \res C$, then
\begin{eqnarray*}
M \models \varphi & \hbox{iff} & \M \models \mathrm{Gr}_{C}(\varphi).
\end{eqnarray*}
\end{theorem}

%\begin{proof}[Proof of Theorem~\ref{grounding}]
\begin{proof}
From left to right, suppose that $\M\models \varphi$. By Proposition \ref{global_grounding}, we know that $\M \models \mathrm{Gr}_{D}(\varphi)$. The result follows since $\sigma(C) \subseteq D$ and
$\mathrm{Gr}_{C}(\varphi)=\mathrm{Gr}_{\sigma(C)}(\varphi)$.
\par
Now, from the right to left direction, take $\varphi=\forall x_1 \forall x_2 \ldots \forall x_n \psi$ a safe splittable temporal sentence and suppose that $\M \models \mathrm{Gr}_{C}(\varphi).$
Again, we can proceed by induction in $n$. If $n=0$, then $\varphi$ is
quantifier free so $\mathrm{Gr}_{C}(\varphi)=\varphi$. Suppose the
result is true for any safe splittable sentence with length up to $n-1$
and put $\varphi= \forall x_1 \alpha(x_1)$ with $\alpha(x_1)=\forall x_2
\ldots \forall x_n \psi(x_1,x_2,\ldots,x_n).$ Notice that $\alpha(x_1)$
is a safe formula that has no more free variables than $x_1$, so, if we apply Lemma \ref{safe&ground&two}, it will be sufficient to show that $\M \models  \textstyle \bigwedge_{c \in C} \alpha(c)$. Since we are supposing that
$$\M \models
\mathrm{Gr}_{C}(\varphi)=\textstyle \bigwedge_{c \in C}
\mathrm{Gr}_{C}(\alpha(c)),$$ and we can apply the induction hypothesis
on any $\alpha(c)$ with $c \in C$, it follows  that $\M \models \textstyle \bigwedge_{c \in C} \alpha(c)$
and this completes the proof.
%\vspace{5cm}
\end{proof}

%%%%%%%%%%%%%%%%%%%%%%%%%%%%%%%%%%%%%%%%%%%%%%%%%%%%%%%%%%%%%%%%%%%%%%%%%%%%%%%%%%%%%%%%%%%

\begin{theorem}
\label{grounding_two}
If $\varphi$ is a safe splittable temporal sentence, then $\langle(D,
\sigma),\T,\T\rangle$ is a temporal equilibrium model of $\varphi$ iff
$\langle(D, \sigma),\T,\T\rangle$ is a temporal equilibrium model of $
\mathrm{Gr}_{C}(\varphi)$.
\end{theorem}

%\begin{proof}[Proof of Theorem~\ref{grounding_two}]
\begin{proof}
Suppose that $\langle(D, \sigma),\T,\T\rangle$ is a temporal equilibrium
model of $\varphi$ and $\langle(D, \sigma),\H,\T\rangle \models
\mathrm{Gr}_{C}(\varphi)$. Since $\varphi$ is safe, we know by Theorem
\ref{th:safe} that $\T = \T \res C$ so, applying Theorem
\ref{grounding}, it follows that $\langle(D, \sigma),\H,\T\rangle
\models \varphi$ and $\H=\T$. This shows that $\langle(D,
\sigma),\T,\T\rangle$ is also a temporal equilibrium model of $
\mathrm{Gr}_{C}(\varphi)$, The other implication follows directly from
the fact that  $\langle(D, \sigma),\H,\T\rangle \models \varphi$ implies
$\langle(D, \sigma),\H,\T\rangle \models \mathrm{Gr}_{C}(\varphi)$.
%\vspace{5cm}
\end{proof}

\begin{proofof}{Theorem~\ref{th:safe_three}}
Let us show that the following assertions are equivalent:
\begin{enumerate}
\item $\langle (D,\sigma),\T,\T\rangle$ is a temporal equilibrium model
of $ \mathrm{Gr}_{C}(\varphi)$
\item $\langle (D,\sigma),\T,\T\rangle$ is a temporal equilibrium model
of $ \varphi$
\item $\langle (D,\sigma),\T,\T\rangle$ is a temporal equilibrium model
of $ \mathrm{Gr}_{C'}(\varphi)$
\end{enumerate}
Taking into account the previous theorem, we only have to prove the
equivalence of $2$ and $3$. Suppose that $\langle
(D,\sigma),\T,\T\rangle$ is a temporal equilibrium model of $ \varphi$
and $\langle (D,\sigma),\H,\T \rangle \models
\mathrm{Gr}_{C'}(\varphi)$. Because of Theorem \ref{th:safe}, we have
that $\T= \T \res C \subseteq \T \res {C'}$ and an obvious extension of
Theorem \ref{grounding} to $C'$, implies that $$\langle (D,\sigma),\H,\T
\rangle \models \varphi$$ and so $\H=\T$. This shows that 2 implies 3.
The other implication ($3. \Longrightarrow 2. $) follows directly.
\end{proofof}
%%%%%%%%%%%%%%%%%%%%%%%%%%%%%%%%%%%%%%%%%%%%%%%%%%%%%%%%%%%%%%%%%%%%%%%%%%%%%%%%%%%%%%%%%%%%%%%%%%%%%%%%%%%%%%%%%%%%%%%%%%%%%%%%%%%%%%%%%%%%%%%%%%%%%%%%%%%
\begin{lemma}
\label{martin}
%\begin{itemize}
%\item $\Pi^{\wedge}$ tiene un \'unico modelo estable $\tuple{J,J}$.
%\item Si $\tuple{T,T}$ es un modelo de {\bf estable} de $\Pi$ y $\tuple{J,J}$ es cualquier modelo de $\Pi^{\wedge}$, entonces $T \subseteq J$.
%\end{itemize}
If $T$ is any equilibrium model of a (non temporal) program $\Pi$ with rules of type~\eqref{f:in0}, then $T \subseteq J$, where $J$ is any model of the normal positive program $\Pi^{\wedge}$.
\end{lemma}
\begin{proof}
We will prove that $\tuple{T \cap J, T}\models \Pi$, and so, $T\cap J=T$ by the minimality of $T$.

Let $B \wedge N \rightarrow H$ of the form \eqref{f:in0} be an arbitrary rule in $\Pi$. To prove $\tuple{T \cap J, T}\models r$ we already know that $\tuple{T,T} \models r$ and remain to prove that if $\tuple{T \cap J, T}\models B \wedge N$ then $\tuple{T \cap J, T} \models H$. So, suppose that $ \tuple{T \cap J , T} \models B \wedge N$. Then $ \tuple{T , T} \models B \wedge N$ y $\tuple{J,J} \models B$. Therefore, $ \tuple{T , T} \models H$ and there exists $p\in H$ such that $ \tuple{T , T} \models p$. Since rule $B \rightarrow p \in \Pi^{\wedge}$ and $\tuple{J,J}\models B$, we get that $\tuple{J,J}\models p$ and so $\tuple{T \cap J , T} \models H$, as we wanted to prove. \end{proof}

Given any rule like $r$ like \eqref{f:in1} of \eqref{f:d} and a set of atoms $X$, we define its \emph{simplification} $simp(r,X)$ as:
\begin{eqnarray*}
simp(r,X) \eqdef \left\{
\begin{array}{l@{\hspace{20pt}}l}
\Next B' \wedge \Next N' \rightarrow \Next H' & \hbox{if } B\subseteq X \ \hbox{and } N \cap X = \emptyset\\
\top & \hbox{otherwise}
\end{array}
\right.
\end{eqnarray*}

\begin{definition}[Slice program]
Given some LTL interpretation $\T$, let us define now the sequence of programs:
\begin{eqnarray*}
slice(\Pi,\T,0) & \eqdef & \Pi^0 = ini_0(\Pi)\\
slice(\Pi,\T,1) & \eqdef & \{ simp(r,T_0) \ | \ r \in ini_1(\Pi) \cup dyn(\Pi) \}\\
slice(\Pi,\T,i+1)  & \eqdef & \{ \Next^i simp(r,T_i) \ | \ r \in dyn(\Pi) \} \hspace{20pt} \hbox{for } i \geq 1
\end{eqnarray*}
\qed
\end{definition}
\begin{theorem}[Theorem 3 in~\cite{ACPV11}]
\label{th:splitseq}
Let $\tuple{\T,\T}$ be a model of a splittable TLP $\Pi$. $\tuple{\T,\T}$ is a temporal equilibrium model of $\Pi$ iff
\begin{itemize}
\item[\rm (i)] $\T^0=T_0$ is a stable model of $slice(\Pi,\T,0)=\Pi^0=ini_0(\Pi)$ and
\item[\rm (ii)] $(\T^1 \setminus At^0)$ is a stable model of $slice(\Pi,\T,1)$ and
\item[\rm (iii)] $(\T^i \setminus At^{i-1})$ is a stable model of $slice(\Pi,\T,i)$ for $i\geq 2$.\qed
\end{itemize}
\end{theorem}

\begin{proofof}{Proposition~\ref{prop:cred_one}}
Let $\tuple{\T,\T}$ be any temporal equilibrium model of $\Pi$ an denote by $\{L_i\}_{i\geq 0}$ the corresponding infinite sequence of ground atoms of $LM(\Pi^{\wedge})$. By Theorem~\ref{th:splitseq}, we know that, for all $i \geq 0$, $T_i$ (resp. $L_i$) is a stable model of $slice(\Pi,\T,i)$ (resp. of $slice(\Pi^{\wedge} ,LM(\Pi^{\wedge}),i)$. Finally, we can apply Lemma \ref{martin} and the fact that $slice(\Pi,\T,0)^{\wedge} = slice(\Pi^{\wedge},LM(\Pi^{\wedge}),0)$ and, for $i \geq 1$, $$slice(\Pi,\T,i)^{\wedge} \subseteq slice(\Pi^{\wedge},LM(\Pi^{\wedge}),i).$$
\end{proofof}

\begin{proofof}{Theorem~\ref{th:delta}}
Let $\tuple{\T,\T}$ be the unique temporal equilibrium model of $\Pi^\wedge$ and let $\tuple{\D,\D}$ denote the temporal interpretation defined by:

\begin{itemize}
\item $D_i=T_i \text{ if } 0 \leq i \leq 1$,
\item $D_2$ is the stable model of the positive non-disjunctive program:
$$\{\Next ^2 B \wedge \Next^2 B' \rightarrow \Next^2 p  \, | \, \nec(B \wedge \Next B'  \rightarrow \Next p) \in dyn(\Pi^{\wedge}) \} \cup slice(\Pi^{\wedge}, \L, 1)$$
\item $D_i=D_2 \text{ if } \geq 3$,
\end{itemize}
It is straightforward to check that $\tuple{\D,\D}$ is a temporal equilibrium model of $\Gamma_{\Pi}$. Notice that $T_2 \subseteq D_2$. This follows from Lemma \ref{martin} and the facts that $\T^2 \setminus AT^1$ is the stable model of $slice(\Pi^{\wedge}, \L, 1)$ and $D_2$ is a model of this latter program.

The cases $i=0,1,2$ follow from Proposition~\ref{prop:cred_one} an the fact that $T_2 \subseteq D_2$.
\par
When $i \geq 3$, we shall prove that $\tuple{\T^i \setminus At^{i-1} \cap \D^i  \setminus At^{i-1}, \T^i  \setminus At^{i-1}}$ is a model of $slice(\Pi^\wedge, \T, i)$ so by Theorem \ref{th:splitseq}, $\T^i \setminus At^{i-1} \cap \D^i  \setminus At^{i-1} = \T^i \setminus At^{i-1}$ and, consequently, $T_i \subseteq D_i$. So, take $\Next^{i} B' \rightarrow \Next^i H' \in slice(\Pi^\wedge, \T, i)$ and suppose that
\begin{eqnarray*}
\tuple{\T^i \setminus At^{i-1} \cap \D^i  \setminus At^{i-1}, \T^i  \setminus At^{i-1}} \models \Next^{i} B'
\end{eqnarray*}

\noindent This fact implies that $\tuple{\T^i \setminus At^{i-1} , \T^i  \setminus At^{i-1}} \models \Next^{i} B'$ and also that there exists a (positive normal) dynamic rule like \eqref{f:d} such that $B \subseteq T_{i-1} \subseteq D_{i-1}$. Since $\tuple{\T^i \setminus At^{i-1} , \T^i  \setminus At^{i-1}}$ is a model of  $slice(\Pi^\wedge, \T, i)$, the only atom $p \in H'$ satisfies $\tuple{\T^i \setminus At^{i-1} , \T^i  \setminus At^{i-1}} \models \Next^i p$. It only rests to show that $\tuple{\D^i, \D^i} \models \Next^i p$ or equivalently $\tuple{\D, \D} \models \Next^2 p$ (notice that $D_i=D_2$ if $i \geq 2$). Finally, we can use that the rule $\Next^2 B \wedge \Next^2 B' \rightarrow \Next ^2 p \in \Gamma_{\Pi}$ and also the fact that $\tuple{\D, \D} \models \Next^2 B \wedge \Next^2 B'$ because $i \geq 3$ and $B' \subseteq D_i=D_2$ and $B \subseteq T_{\i-1} \subseteq D_{i-1}=D_2$.
\end{proofof}

\end{document}